\documentclass[a4,11pt]{article}

\usepackage{geometry, fancybox, framed}
\usepackage{enumerate}
\usepackage{amsmath,amssymb, bm}
\usepackage{mathtools}
\usepackage{graphics}
\usepackage{multirow}

\newtheorem{proof}{proof}

\newtheorem{thm}{Theorem}
\newtheorem{lemma}{Lemma}

\graphicspath{{fig/}}
\allowdisplaybreaks[4]

\usepackage{my-style}

\begin{document}
\title{EB-RANSAC: Random Sample Consensus based on Energy-Based Model}
\author{
\large{Muneki Yasuda}\thanks{Graduate School of Science and Engineering, Yamagata University} 
\and \large{Nao Watanabe}\thanks{TSCSK Corporation}
\and \large{Kaiji Sekimoto}\footnotemark[1]
}
\date{}
\maketitle
%%%%%%%%%%%%%%%%%%%%%%%%%%%%%%%%%%%%%%%%

\begin{abstract}
Random sample consensus (RANSAC), which is based on a repetitive sampling from a given dataset, is one of the most popular robust estimation methods.
In this study, an energy-based model (EBM) for robust estimation that has a similar scheme to RANSAC, energy-based RANSAC (EB-RANSAC), is proposed. 
EB-RANSAC is applicable to a wide range of estimation problems similar to RANSAC. 
However, unlike RANSAC, EB-RANSAC does not require a troublesome sampling procedure and has only one hyperparameter. 
The effectiveness of EB-RANSAC is numerically demonstrated in two applications: a linear regression and maximum likelihood estimation.
\end{abstract}

\section{Introduction}

Consider an estimation problem formulated by the minimization of a loss function based on a given dataset. 
Datasets frequently include outliers that have a negative effect on estimators. 
A robust estimation method is needed to reduce the effect of the outliers. 
M-estimator~\cite{Menezes2021} is one of the most popular robust estimation methods.  
In the M-estimator method, a robust loss function,   
such as Huber, Tukey's biweight, and logcosh losses, is employed instead of the original loss function.  
M-estimator methods enable the robust estimation problem to be solved using a deterministic approach, that is, the minimization of the robust loss function;
thus they have repeatability. 
However, it is not trivial to determine the robust loss function that is suitable for a specific problem. 
Random sample consensus (RANSAC)~\cite{Fischler1981} is a popular robust estimation method. 
In RANSAC, a small subset, known as hypothetical-inlier set, is randomly selected from a given dataset, 
and subsequently, a temporary estimator is obtained using the hypothetical-inlier set. 
The probability that the hypothetical-inlier set includes outliers is expected to be extremely low because the size of the hypothetical-inlier set is small.
RANSAC repeats this sequence (i.e., the subset selection and estimation) until the best estimator is obtained. 
RANSAC is straightforward to implement and can be applied systematically to a wide range of estimation problems.  
However, RANSAC does not have repeatability because it is based on the sampling, and furthermore, it has several hyperparameters.

In this study, we propose an energy-based model (EBM) that has a similar mechanism to that of RANSAC.
The proposed robust estimator can be obtained from the maximization of the EBM. 
In this paper, the proposed approach is referred to as \textit{energy-based RANSAC} (EB-RANSAC). 
Compared with RANSAC, the proposed EB-RANSAC has several advantages: 
(a) it does not require a sampling procedure for selecting the hypothetical-inlier set (thus, it has repeatability) 
and (b) it has only one hyperparameter. 
The remainder of this paper is organized as follows: 
Section \ref{sec:RANSAC} provides a brief explanation of RANSAC. 
Section \ref{sec:proposed_model} proposes EB-RANSAC.
An EBM which forms a foundation of EB-RANSAC  is proposed in Section \ref{eqn:EBM},  
and the relationship between the EBM and RANSAC is discussed in Section \ref{sec:conditional-distribution_RBM}. 
The EB-RANSAC estimator is proposed in Section \ref{sec:EBM-RANSAC_estimator}. 
Section \ref{sec:theory} provides a theoretical analysis for the EB-RANSAC estimator in the maximum likelihood (ML) scenario. 
Section \ref{sec:experiments} demonstrates numerical experiments of the linear regression problem and the ML estimation.
Finally, the conclusion is provided in Section \ref{sec:conclusion}. 
This paper is an extended version of \cite{Yasuda2025}, which includes new theoretical and experimental results.

\section{Robust Estimation based on RANSAC}
\label{sec:RANSAC}

Consider a model $f_{\theta}(\bm{x})$ that takes $\bm{x}:= \{x_i  \mid i = 1,2,\ldots, n\}$ as an input 
and outputs $y$. Here, $y$ can be both a scalar and vector; $\theta$ denotes the set of model parameters. 
When the model is a probabilistic model, $\bm{x}$ are random variables and $y$ is a scalar representing probability density (or mass).
Suppose that the loss function of the model for a given dataset $D:=\{\mbf{d}^{(\mu)} \mid \mu = 1,2,\ldots, N\}$ 
is expressed as follows:
\begin{align}
L(\theta; D):= \frac{1}{N}\sum_{\mu=1}^N \ell(\theta; \mbf{d}^{(\mu)}),
\label{eqn:loss_function}
\end{align}
where $\ell(\theta; \mbf{d}^{(\mu)})$ is the loss for the $\mu$th data point (it may include a regularizer for $\theta$). 
$\ell(\theta; \mbf{d}^{(\mu)})$ may be a squared error when the model is a deterministic regression model, 
or may be a negative log likelihood when the model is a probabilistic model.
In a supervised-training scenario, the data points include input and output information, that is, $\mbf{d}^{(\mu)}= \{ \mbf{x}^{(\mu)} ,\mrm{y}^{(\mu)}\}$, 
whereas in an unsupervised-training scenario, they only include input information, that is, $\mbf{d}^{(\mu)}=  \mbf{x}^{(\mu)}$.
The model parameters can be trained through the minimization of the loss function with respect to $\theta$. 

Suppose that the dataset includes few outliers $D_{\mrm{out}} \subset D$ (i.e., inliers are dominant in $D$). 
The model parameters trained by the minimization of equation (\ref{eqn:loss_function}) can be poor owing to the influence of the outliers. 
Therefore, a robust-estimation method is required in such a case. 
The standard RANSAC consists of the following procedure~\cite{Fischler1981,Raguram2013}. 
Step 1: a fixed-size small subset $\hat{D} \subset D$ is randomly selected, which is known as the \textit{hypothetical-inlier set}.
Subsequently, the model parameters are trained using the hypothetical-inlier set, $\hat{\theta} = \argmin_{\theta} L(\theta; \hat{D})$. 
Step2: all data points in $D$ are tested against the trained model with $\hat{\theta}$, 
and all data points $\{\mbf{d}^{(\mu)} \}$ that satisfy $T_{\mrm{cons}} > \ell(\hat{\theta}; \mbf{d}^{(\mu)})$ for a fixed threshold $T_{\mrm{cons}}$
are regarded as the \textit{consensus set} 
(in other words, the data points belonging to the consensus set adequately fit to the trained model). 
Step 3: if the size of the consensus set is larger than a certain threshold, 
then the loss value of the trained model is evaluated for the consensus set, 
and the obtained loss value is taken as the score of the model.
Steps 1--3 are repeated until the best model in terms of the score is obtained. 
In practice, Steps 1--3 are repeated a certain number of times, 
and a model with the best score is selected from all the models obtained in the iteration. 
Additionally, locally optimized RANSAC (LO-RANSAC) is proposed~\cite{Chum2003}. 
In LO-RANSAC, the trained model obtained using the hypothetical-inlier set is retrained using the consensus set.

\section{Proposed Method: EB-RANSAC}
\label{sec:proposed_model}

\subsection{Energy-based model}
\label{eqn:EBM}

In this section, we propose an EBM that has a similar mechanism to that of RANSAC.
Consider a joint distribution expressed as follows:
\begin{align}
P(\theta, \bm{w} \mid D, \beta):=\frac{1 }{Z(D,\beta)} 
\exp\bigg(- \sum_{\mu=1}^N w_{\mu}\ell(\theta; \mbf{d}^{(\mu)}) + \beta \sum_{\mu=1}^N w_{\mu}\bigg),
\label{eqn:proposed_model-joiint}
\end{align}
where $\bm{w}:= \{w_{\mu} \mid \mu = 1,2,\ldots, N\} \in \{0, 1\}^N \setminus \bm{0}$, 
which means that $w_{\mu}$s are binary random variables and at least one of which takes one. 
$Z(D,\beta)$ is the normalization constant, 
\begin{align*}
Z(D,\beta ):= \int_{\mcal{X}_{\theta}}  \bigg[\sum_{\bm{w} }
\exp\bigg(- \sum_{\mu=1}^N w_{\mu}\ell(\theta; \mbf{d}^{(\mu)}) + \beta \sum_{\mu=1}^N w_{\mu}\bigg) \bigg] d\theta,
\end{align*}
where $ \int_{\mcal{X}_{\theta}} d\theta$ represents multiple integration over $\theta$ in which $\mcal{X}_{\theta}$ denotes the sample space of $\theta$, 
and $\sum_{\bm{w}}$  the multiple summation over the sample space of the assigned discrete random variables, 
that is, $\sum_{\bm{w}}=\sum_{\bm{w} \in \{0,1\}^N\setminus \bm{0}}$. We assume that the normalization constant is finite.
Equation (\ref{eqn:proposed_model-joiint}) can be considered as an EBM, 
$P(\theta, \bm{w} \mid D, \beta) \propto \exp(- E(\theta, \bm{w}; D,\beta) )$, whose energy function is expressed as follows:
\begin{align*}
E(\theta, \bm{w}; D,\beta):=\sum_{\mu=1}^N w_{\mu}\ell(\theta; \mbf{d}^{(\mu)}) - \beta \sum_{\mu=1}^N w_{\mu}.
\end{align*}
If $w_{\mu} = 0$, the influence of the $\mu$th data point vanishes in equation (\ref{eqn:proposed_model-joiint});  
thus, the realizations of $\bm{w}$ correspond to \textit{select} or \textit{not select} of the corresponding data points. 
The case $\bm{w} = \bm{0}$, that is, the case in which no data point is selected, is not considered because it is eliminated from the sample space of $\bm{w}$. 
The hyperparameter $\beta \in \mbb{R}$, which acts as a constant bias for $\bm{w}$, can be considered as the threshold $T_{\mrm{cons}}$ used to determine the consensus set in RANSAC, 
which will be revealed later.  

The expression of equation (\ref{eqn:proposed_model-joiint}) appears similar to that of the reweighted probabilistic model (RPM)~\cite{Wang2017} but several differences exist.
In the RPM, (a) $\bm{w}$ are continuous (and greater than zero) and (b) prior distributions of $\theta$ and $\bm{w}$ have to be explicitly assumed.
If those differences are ignored, the proposed joint distribution may be considered as a special case of the RPM. 
Unlike the RPM, $\bm{w}$ are the binary random variables in the proposed model. 
This makes a marginal operation on equation (\ref{eqn:proposed_model-joiint}) efficient, as demonstrated  in Section \ref{sec:EBM-RANSAC_estimator}.

\subsection{Conditional distributions}
\label{sec:conditional-distribution_RBM}

The joint distribution of equation (\ref{eqn:proposed_model-joiint}) has a similar scheme to that of RANSAC. 
Conditional distributions of the joint distribution are considered to observe this similarity.

For a fixed $\bm{w}$, the conditional distribution of $\theta$ is expressed as follows:
\begin{align}
P(\theta \mid \bm{w}, D, \beta) \propto \exp\bigg(- \sum_{\mu=1}^N w_{\mu}\ell(\theta; \mbf{d}^{(\mu)}) \bigg).
\label{eqn:cond_dis-theta}
\end{align}
The maximization of the conditional distribution with respect to $\theta$ corresponds to the training with a subset of $D$ 
that comprises the data points with $w_{\mu} = 1$. 
This maximization can be considered as the training using the hypothetical-inlier set (cf. Step 1 in RANSAC in Section \ref{sec:RANSAC}).

Alternatively, for a fixed $\theta$, the conditional distribution of $\bm{w}$ is expressed as follows:
\begin{align}
P(\bm{w} \mid \theta, D, \beta) = \frac{1}{\Psi(\theta, D, \beta)} 
\prod_{\mu=1}^N \exp\big\{\big(\beta - \ell(\theta; \mbf{d}^{(\mu)}) \big)w_{\mu} \big\},
\label{eqn:cond_dis-w}
\end{align}
where  $\Psi(\theta, D, \beta)$ is the normalization constant defined as follows:
\begin{align}
\Psi(\theta, D,\beta):=\sum_{\bm{w}}\prod_{\mu=1}^N \exp\big\{\big(\beta - \ell(\theta; \mbf{d}^{(\mu)})\big)w_{\mu} \big\}=\prod_{\mu=1}^N\Big(1 + \exp\big(\beta - \ell(\theta; \mbf{d}^{(\mu)}) \big)\Big) - 1.
\label{eqn:Psi_beta}
\end{align}
In the above derivation, we used the equation: 
\begin{align}
\sum_{\bm{w}} f(\bm{w}) = \sum_{w_1 \in \{0,1\}}\sum_{w_2 \in \{0,1\}} \cdots \sum_{w_N \in \{0,1\}}f(\bm{w}) - f(\bm{0}).
\label{eqn:sum_over_w_formulation}
\end{align}
The expectation of $w_{\mu}$ over the conditional distribution in equation (\ref{eqn:cond_dis-w}) is
\begin{align}
\mathbb{E}[w_{\mu} \mid \theta, D, \beta]:= \sum_{\bm{w}}w_{\mu}P(\bm{w} \mid \theta, D, \beta)
=\frac{\Psi(\theta, D, \beta) + 1}{\Psi(\theta, D, \beta)}
\sig \big(\beta - \ell(\theta; \mbf{d}^{(\mu)}) \big),
\label{eqn:expect_w}
\end{align}
where $\sig(z) := 1 / (1 + e^{-z})$ is the sigmoid function, and we used equation (\ref{eqn:sum_over_w_formulation}). 
In the above derivation, we used equation (\ref{eqn:Psi_beta}). 
This expectation can be read as the probability that $w_{\mu}$ takes one (in other words, the probability that the $\mu$th data point is selected), 
because 
\begin{align}
P(w_{\mu} = 1 \mid \theta, D,\beta) &= \sum_{\bm{w}}\delta(w_{\mu}, 1)P(\bm{w} \mid \theta, D, \beta)
=\sum_{\bm{w}}w_{\mu}P(\bm{w} \mid \theta, D, \beta)
=\mathbb{E}[w_{\mu} \mid \theta, D, \beta],
\label{eqn:P(w=1)}
\end{align}
where $\delta(a,b)$ is the Kronecker delta function. In the above derivation, we used $\delta(w_{\mu}, 1) = w_{\mu}$. 
From equations (\ref{eqn:expect_w}) and (\ref{eqn:P(w=1)}), assuming $\Psi(\theta, D, \beta)  \gg 1$ (which implies $N \gg 1$) leads to 
\begin{align}
P(w_{\mu} = 1 \mid \theta, D, \beta)\approx \sig \big(\beta - \ell(\theta; \mbf{d}^{(\mu)}) \big).
\label{eqn:P(w=1)_approx}
\end{align}
Suppose that $\theta$ is the estimator obtained based on a hypothetical-inlier set,  
equation (\ref{eqn:P(w=1)_approx}) can be regarded as the probability that the $\mu$th data point belongs to the corresponding consensus set. 
This probability is high (higher than $1/2$) when $\beta > \ell(\theta; \mbf{d}^{(\mu)})$. 
Therefore, the conditional distribution in (\ref{eqn:cond_dis-w}) is considered as the probability of the consensus set 
in which $\beta$ acts as the threshold  $T_{\mrm{cons}}$ (cf. Step 2 in RANSAC in Section \ref{sec:RANSAC}). 
This correspondence is clearer at the maximum point of $P(\bm{w} \mid \theta, D, \beta)$, 
$\hat{\bm{w}}=\argmax_{\bm{w}} P(\bm{w} \mid \theta, D, \beta)$.
Assuming that $\mu$ exists such that $\beta > \ell(\theta; \mbf{d}^{(\mu)})$, 
from equation (\ref{eqn:cond_dis-w}),
\begin{align*}
\hat{w}_{\mu} = 
\begin{cases}
1 & \beta > \ell(\theta; \mbf{d}^{(\mu)})\\
0 & \beta \leq \ell(\theta; \mbf{d}^{(\mu)})
\end{cases}
\quad (\mu = 1,2,\ldots, N)
\end{align*}
is obtained, meaning $\hat{\bm{w}}$ corresponds to the configuration of the consensus set for given $\theta$.

Now, let us consider a Markov chain Monte Carlo process on the joint distribution in equation (\ref{eqn:proposed_model-joiint}) 
based on the alternative sampling using the conditional distributions in equations (\ref{eqn:cond_dis-theta}) and (\ref{eqn:cond_dis-w}). 
The states of $\theta$ and $\bm{w}$ repeats the stochastic transition according to the conditional distributions  
and reaches the stationary distribution, equation (\ref{eqn:proposed_model-joiint}), after an sufficient number of transitions.  
Here, we consider an extreme case in which the transitions of $\theta$ and $\bm{w}$ are deterministically selected by the maximum points of both conditional distributions
\footnote{This extreme situation can be regarded as a special case of the stochastic process. 
Consider an extended joint distribution by adding temperature $T > 0$ as $P(\theta, \bm{w} \mid D, \beta, T) \propto P(\theta, \bm{w} \mid D, \beta)^{1/T}$. 
This extension is used, for instance, in simulated annealing~\cite{Kirkpatrick1983}. In the limit of $T \to 0$, the alternative sampling process becomes the considering extreme case.}. 
In this case, the states will converge at a local maximum point of the joint distribution of equation (\ref{eqn:proposed_model-joiint}). 
In such a situation, we assume that process starts from an initial state $\bm{w}_0$, which is randomly selected (i.e., the hypothetical-inlier set).
Based on $\bm{w}_0$, state $\theta_0$ is determined according to the conditional distribution in equation (\ref{eqn:cond_dis-theta}), 
that is, $\theta_0 = \argmax_{\theta}P(\theta \mid \bm{w}_0, D, \beta)$,    
which is the training solution based on the hypothetical-inlier set, $\bm{w}_0$, as previously mentioned.
Next, based on $\theta_0$, $\bm{w}_1$ is determined according to the conditional distribution in equation (\ref{eqn:cond_dis-w}), 
that is, $\bm{w}_1 = \argmax_{\bm{w}}P(\bm{w} \mid \theta_0, D, \beta)$,  
which is the consensus set for $\theta_0$ as previously mentioned.   
Subsequently, $\theta_1 = \argmax_{\theta}P(\theta \mid \bm{w}_1, D, \beta)$ corresponds to the training solution based on the consensus set, 
and $\bm{w}_2 = \argmax_{\bm{w}}P(\bm{w} \mid \theta_1, D, \beta)$ then corresponds to the modified consensus set based on $\theta_1$. 
The states will converge at a stationary point corresponding to a local maximum of the joint distribution by repeating this sequence. 
This process appears similar to that of the LO-RANSAC instead of the original RANSAC. 
Therefore, local maximum points in the joint distribution can be considered as the LO-RANSAC estimators obtained from different hypothetical-inlier sets.
However, unlike RANSAC, the size of the consensus set is not explicitly considered in the process (cf. Step 3 in RANSAC in Section \ref{sec:RANSAC}).

\subsection{EB-RANSAC estimator}
\label{sec:EBM-RANSAC_estimator}

In RANSAC, the sampling of the hypothetical-inlier set is repeated numerous times. 
When the size of the hypothetical-inlier set is $k$, the number of the possible hypothetical-inlier sets is $\binomial(N,k)$, here, $\binomial$ denotes the binomial coefficient.  
The best hypothetical-inlier set must be drawn from the huge set of candidates to obtain the best model, and this can be costly.
Progressive sample consensus~\cite{Chum2005} reduces the number of iterations based on qualities of data points that are defined a quality function. 

In the proposed EBM, such troublesome iteration can be skipped by using a marginalizing operation. 
Marginalizing out $\bm{w}$ from equation (\ref{eqn:proposed_model-joiint}) yields the following:
\begin{align}
P(\theta \mid D, \beta) &= \sum_{\bm{w}}P(\theta, \bm{w} \mid D, \beta)
=\frac{1 }{Z(D, \beta)} \bigg[ \prod_{\mu=1}^N \sum_{w_{\mu} \in \{0,1\}} \exp\big\{\big(\beta - \ell(\theta; \mbf{d}^{(\mu)}) \big)w_{\mu} \big\}
-1\bigg]\nn
&=\frac{1 }{Z(D, \beta)} \bigg[ \prod_{\mu=1}^N\Big(1 + \exp\big(\beta - \ell(\theta; \mbf{d}^{(\mu)}) \big)\Big) 
-1\bigg]\nn
&=\frac{1 }{Z(D, \beta)} \bigg[\exp\bigg( \sum_{\mu=1}^N \sfp \big(\beta - \ell(\theta; \mbf{d}^{(\mu)}) \big)\bigg) - 1\bigg],
\label{eqn:marginal_dist}
\end{align}
where $\sfp(z):= \ln(1 + e^z)$ is the softplus function, and we used equation (\ref{eqn:sum_over_w_formulation}).  
Owing to the summation over all possible realizations of $\bm{w}$, this marginal distribution can consider all possible hypothetical-inlier sets (without fixed size). 
The proposed EB-RANSAC estimator, $\theta^*$, is obtained through the maximization of equation (\ref{eqn:marginal_dist}). 
Thus, EB-RANSAC can be formulated as a minimization problem as follows: 
\begin{align}
\theta^* =\argmin_{\theta} L_{\mrm{ER}}(\theta; D, \beta),
\label{eqn:EBM-RANSAC}
\end{align}
where
\begin{align}
L_{\mrm{ER}}(\theta; D, \beta):=-\frac{1}{N}\sum_{\mu=1}^N \sfp \big(\beta - \ell(\theta; \mbf{d}^{(\mu)}) \big)
\label{eqn:EBM-RANSAC-Loss}
\end{align}
is the loss function of EB-RANSAC. In this paper, we refer to equation (\ref{eqn:EBM-RANSAC-Loss}) as EB-RANSAC loss.
In contrast to RANSAC
\footnote{The original RANSAC has several hyperparameters, for instance, the size of the hypothetical-inlier set, number of iterations, 
and thresholds for the loss value and for the size of the consensus set.}, the proposed robust estimator requires no sampling and has only one hyperparameter $\beta$.
EB-RANSAC can be applied to any loss function that takes the form of equation (\ref{eqn:loss_function}).
As $\beta$ increases, the EB-RANSAC estimator is asymptotically consistent with 
that obtained by minimizing equation (\ref{eqn:loss_function}), 
because when $\beta \gg 1$,
\begin{align*}
\sfp \big(\beta - \ell(\theta; \mbf{d}^{(\mu)}) \big) = \beta - \ell(\theta; \mbf{d}^{(\mu)}) +O(e^{-\beta})
\approx - \ell(\theta; \mbf{d}^{(\mu)})  + \mrm{constant}.
\end{align*}

EB-RANSAC has a drawback:  
in general, the EB-RANSAC loss in equation (\ref{eqn:EBM-RANSAC-Loss}) can be non-convex with respect to $\theta$ 
when the original minimization problem, that is, the minimization of equation (\ref{eqn:loss_function}), is a convex optimization problem (cf. Section \ref{sec:MLE_Exponential}).

\section{Theoretical Analysis of EB-RANSAC}
\label{sec:theory}

In this section, we explain the reason why EB-RANSAC can achieve the robust estimation from a theoretical perspective. 
Consider the case in which the model $f_{\theta}(\bm{x}) = p_{\theta}(\bm{x})$ is a probabilistic model  with discrete random variables 
and the loss is the negative log likelihood: $\ell(\theta;\mbf{d}^{(\mu)}) = - \ln p_{\theta}(\mbf{x}^{(\mu)})$, where $\mbf{d}^{(\mu)} = \mbf{x}^{(\mu)}$, 
and assume that $p_{\theta}(\bm{x})$ can express any distribution. 
In this case, the loss function in equation (\ref{eqn:loss_function}) becomes 
\begin{align*}
L(\theta; D) =-\sum_{\bm{x}}q(\bm{x}) \ln p_{\theta}(\bm{x}) ,
\end{align*}
where 
\begin{align*}
q(\bm{x}):=\frac{1}{N}\sum_{\mu=1}^N \delta\big(\mbf{x}^{(\mu)}, \bm{x}\big)
\end{align*}
is the empirical distribution (or the data distribution) of $D$. 
Minimizing the loss function with respect to $p_{\theta}(\bm{x})$ leads to $p_{\theta}(\bm{x}) = q(\bm{x})$;  
this is the standard ML scenario. 

\begin{figure}[htb]
\centering
\includegraphics[height=5cm]{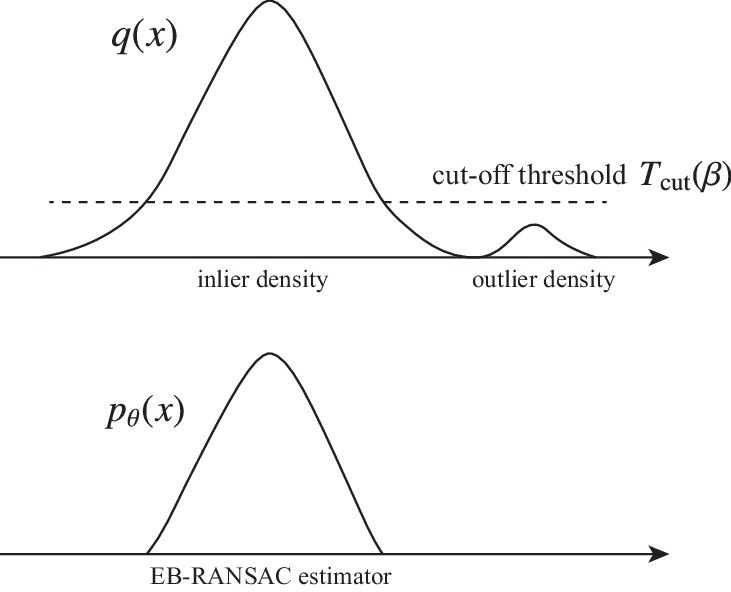}
\caption{Illustration of the EB-RANSAC estimator in equation (\ref{eqn:solution_EBM-RANSAC_MLE}).
Top panel: the empirical distribution of $D$ that has outliers' probability in its tail. 
Bottom panel: the EB-RANSAC estimator obtained by removing low probability region (lower than the cut-off threshold)  from the empirical distribution.}
\label{fig:EBM-RANSAC}
\end{figure}

We need to determine $p_{\theta}(\bm{x})$ that minimizes the EB-RANSAC loss, 
\begin{align}
L_{\mrm{ER}}(\theta; D, \beta)=-\sum_{\bm{x}} q(\bm{x})\sfp \big(\beta + \ln p_{\theta}(\bm{x}) \big),
\label{eqn:EBM-RANSAC_MLE}
\end{align}
to clarify the property of the EB-RANSAC estimator.
For this minimization problem, we have the following theorem: 
\begin{thm}
Given $q(\bm{x})$, the distribution $p_{\theta}(\bm{x})$ that minimizes equation (\ref{eqn:EBM-RANSAC_MLE}) is as follows:
\begin{align}
p_{\theta}(\bm{x}) = \frac{e^{-\beta}}{T_{\mrm{cut}}(\beta)}\relu\big( q(\bm{x}) -  T_{\mrm{cut}}(\beta)\big),
\label{eqn:solution_EBM-RANSAC_MLE}
\end{align}
where $\relu(x):= \max(x, 0)$ is the rectified linear unit and $T_{\mrm{cut}}(\beta)$ is the solution to 
\begin{align}
T_{\mrm{cut}}(\beta) =  e^{-\beta}\sum_{\bm{x}}\relu\big( q(\bm{x}) -  T_{\mrm{cut}}(\beta)\big).
\label{eqn:Tcut_equation}
\end{align}
\label{thm:minimum_solution}
\end{thm}
\noindent
The proof of this theorem is described in Appendix \ref{app:proof_minimum-solution}. 
From equation (\ref{eqn:solution_EBM-RANSAC_MLE}), 
the EB-RANSAC estimator, $p_{\theta}(\bm{x})$, is obtained by cutting off low probability region (lower than the cut-off threshold $T_{\mrm{cut}}(\beta)$) from $q(\bm{x})$.
The outliers would be located in the tail of $q(\bm{x})$ and their probability would be extremely small compared with that of the inliers (see the top panel in figure \ref{fig:EBM-RANSAC}). 
The probability of the outliers is removed from the EB-RANSAC estimator when the cut-off threshold is higher than outliers' probability (see the bottom panel in figure \ref{fig:EBM-RANSAC}). 
In the EB-RANSAC estimator, the tails of the inliers' probability are also removed simultaneously. 

The cut-off threshold $T_{\mrm{cut}}(\beta)$ plays an important role in EB-RANSAC. 
From equation (\ref{eqn:Tcut_equation}), $T_{\mrm{cut}}(\beta)$ corresponds to the root(s) of $h_{\beta}(T):=T - e^{-\beta}b(T)$, 
where $b(T):=\sum_{\bm{x}}\relu( q(\bm{x}) -  T)$.
The following theorem explains the behavior of the cut-off threshold: 
\begin{thm}
Given $q(\bm{x})$, the followings are ensured:  
(i) for a fixed $\beta$,  $h_{\beta}(T)$ has a unique root in the range $(0, T^*)$, where $T^* := \max_{\bm{x}}q(\bm{x})$, 
(ii) $T_{\mrm{cut}}(\beta)$ is a monotonic decreasing function for $\beta$, 
and (iii) $\lim_{\beta \to \infty}T_{\mrm{cut}}(\beta) = 0$ and $\lim_{\beta \to -\infty}T_{\mrm{cut}}(\beta) = T^*$.
\label{thm:T(beta)}
\end{thm}
\noindent
The proof of this theorem is described in Appendix \ref{app:proof_T(beta)}. 
From this theorem, the cut-off threshold monotonically increases within the range $(0, T^*)$ as $\beta$ decreases.
As discussed in section \ref{sec:EBM-RANSAC_estimator}, the EB-RANSAC estimator is asymptotically consistent with 
that obtained from the standard loss function minimization as $\beta$ increases. 
This can be reconfirmed by considering the limit of $\beta \to \infty$. 
From Theorem \ref{thm:T(beta)}(iii), $T_{\mrm{cut}}(\beta) \to 0$ in this limit, 
implying that no region is removed and $p_{\theta}(\bm{x}) = q(\bm{x})$, that is, the ML estimator is obtained. 
Alternatively, from Theorem \ref{thm:T(beta)}(iii), $T_{\mrm{cut}}(\beta) \to T^*$ in the limit of $\beta \to -\infty$. 
This implies that $p_{\theta}(\bm{x})$ will be a strongly localized distribution (extremely, a delta function) at the maximum point of $q(\bm{x})$
when $\beta$ is negatively large (this can be observed in figure \ref{fig:MLE_Gauss_beta} in Section \ref{sec:MLE_Gaussian}).

Although the theoretical analysis in this section is based on a discrete distribution, 
we expect to obtain a similar result for a continuous distribution.

\section{Numerical Experiments}
\label{sec:experiments}

In this section, we demonstrate EB-RANSAC using numerical experiments. 
In the EB-RANSAC method, we numerically minimized the EB-RANSAC loss in equation (\ref{eqn:EBM-RANSAC-Loss}) using a gradient descent method.    
We employed the solution that achieved the minimum loss among 30 parallel runs that started from different initial values as the EB-RANSAC estimator.

\subsection{Application to linear regression problem}
\label{sec:linear-regression}

\begin{figure}[htb]
\centering
\includegraphics[height=4cm]{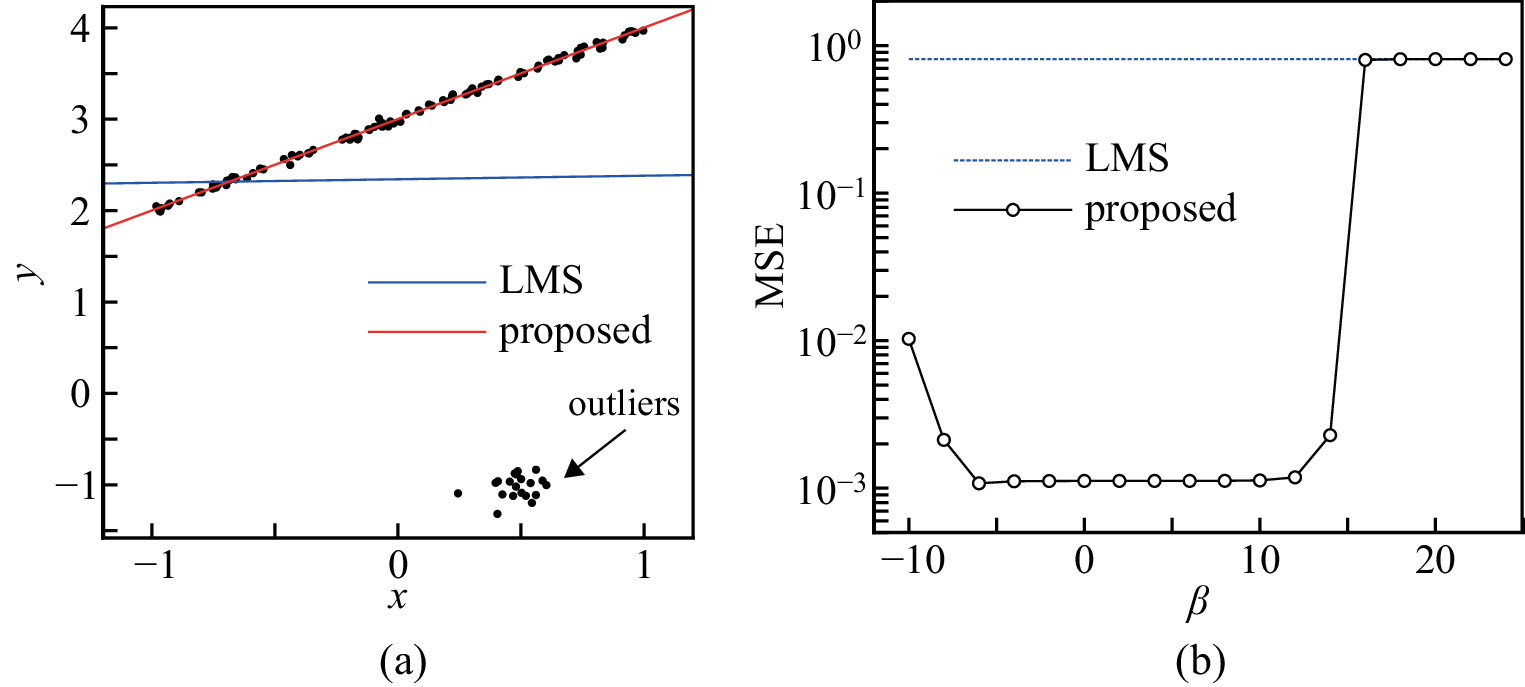}
\caption{(a) Estimators obtained from the EB-RANSAC with $\beta = 5$ and standard LMS methods. (b) MSE plot for various $\beta$s. 
The blue broken line represents the MSE of the LMS estimator.}
\label{fig:LMS}
\end{figure}

Consider a linear regression problem in which $y = f_{\theta}(x) = a x + b$, where $\theta = \{ a,b\}$.  
For a dataset $D$ with $\mbf{d}^{(\mu)} = \{ \mrm{x}^{(\mu)}, \mrm{y}^{(\mu)}\}$, the loss is defined by a squared error:  
$\ell(\theta;\mbf{d}^{(\mu)}) = (\mrm{y}^{(\mu)} - a \mrm{x}^{(\mu)} - b )^2$.
Figure \ref{fig:LMS}(a) depicts the result obtained from the EB-RANSAC method. 
For a comparison, the result obtained from the least mean square (LMS) (i.e., the minimization of  equation (\ref{eqn:loss_function})) is also plotted. 
In this experiment, the inliers ($100$ points) were generated by adding small noise to points on $y = x + 3$ 
and the outliers ($20$ points) were drawn from a Gaussian distribution.
The EB-RANSAC estimator exhibits an excellent fit to the inliers. 

Figure \ref{fig:LMS}(b) depicts the plot of the mean squared error (MSE) between the true parameters (i.e., the parameters of the inliers) and the estimated parameters 
obtained from the EB-RANSAC method for various values of $\beta$, 
in which the true parameters are $a_{\mrm{true}} = 1$ and $b_{\mrm{true}} = 3$. 
For a comparison, the MSE of the LMS estimator is also plotted (the blue broken line). 
An appropriate range of $\beta$ is present. 
As discussed in Sections \ref{sec:EBM-RANSAC_estimator} and \ref{sec:theory},  
the EB-RANSAC estimator approaches the LMS estimator as $\beta$ increases. 
As stated in Theorem \ref{thm:T(beta)}(ii), the cut-off threshold, $T_{\mrm{cut}}(\beta)$, is a monotonic decreasing function with respect to $\beta$. 
Therefore, when the value of $\beta$ is too small, the cut-off threshold will be too high, resulting in the exclusion of some inliers. 
For this reason, the MSE degradation is observed  for small values of $\beta$. 
A similar behavior can be observed in the numerical result in figure \ref{fig:MLE_Gauss}(b) in the next section.

\subsection{Application to maximum likelihood estimation}
\label{sec:MLE}

\subsubsection{Gaussian distribution}
\label{sec:MLE_Gaussian}

\begin{figure}[htb]
\centering
\includegraphics[height=4cm]{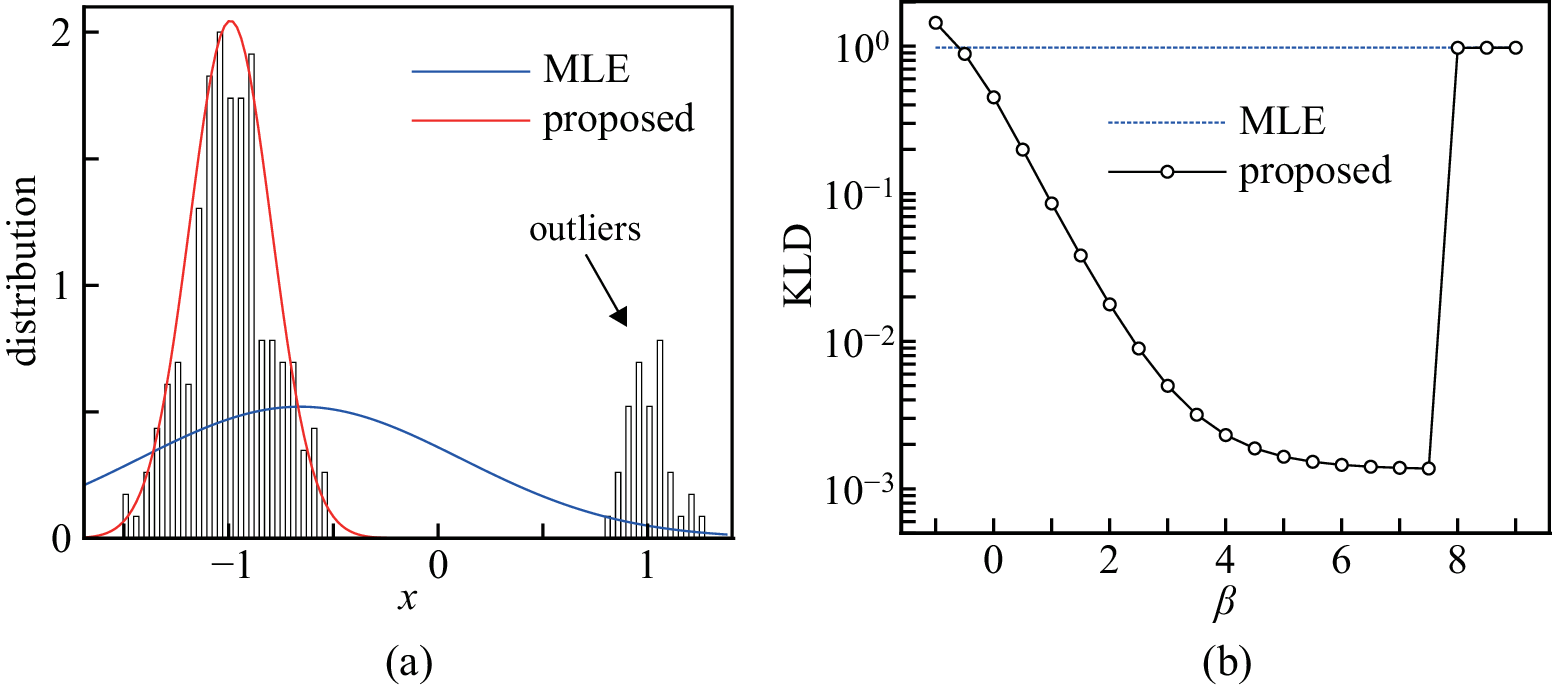}
\caption{(a) Gaussians obtained from the EB-RANSAC with $\beta = 5$ and standard MLE methods. 
The bars represent the histogram of the dataset that is normalized such that the maximum value is equal to two for easier comparison with the Gaussians. 
(b) KLD between the EB-RANSAC estimator and the inlier Gaussian. The blue broken line represents the KLD of the ML estimator.}
\label{fig:MLE_Gauss}
\end{figure}

Consider the ML estimation (MLE) of a univariate Gaussian: 
\begin{align*}
f_{\theta}(x) = \mcal{N}(x \mid m, \sigma^2):=
\frac{1}{\sqrt{2\pi \sigma^2}}\exp\bigg(-\frac{(x - m)^2}{2\sigma^2}\bigg),
\end{align*}  
where $\theta = \{m, \sigma\}$.
For a dataset $D$ with $\mbf{d}^{(\mu)} = \mrm{x}^{(\mu)}$, 
the loss is expressed based on the negative log likelihood:  
$\ell(\theta;\mbf{d}^{(\mu)}) = -\ln f_{\theta}(\mrm{x}^{(\mu)})$.
The standard MLE is obtained from the minimization of equation (\ref{eqn:loss_function}). 
Figure \ref{fig:MLE_Gauss}(a) depicts the estimators obtained from the EB-RANSAC and standard MLE methods. 
The inliers (200 points) were drawn form a Gaussian with $m = -1$ and $\sigma = 0.2$, 
and the outliers (40 points) were drawn form a Gaussian with $m = 1$ and $\sigma = 0.1$. 
The EB-RANSAC estimator fits well into the inlier distribution. 
Figure \ref{fig:MLE_Gauss}(b) depicts the Kullback--Leibler divergence (KLD) between the EB-RANSAC estimator 
and inlier distribution, that is, $\mcal{N}(x \mid -1, (0.2)^2)$, 
in which the blue broken line represents the KLD between the ML estimator and inlier distribution. 
Similar to the result obtained in Section \ref{sec:linear-regression}, an appropriate range of $\beta$ is present
and the EB-RANSAC estimator approaches the ML estimator as $\beta$ increases.

\begin{figure}[htb]
\centering
\includegraphics[height=4cm]{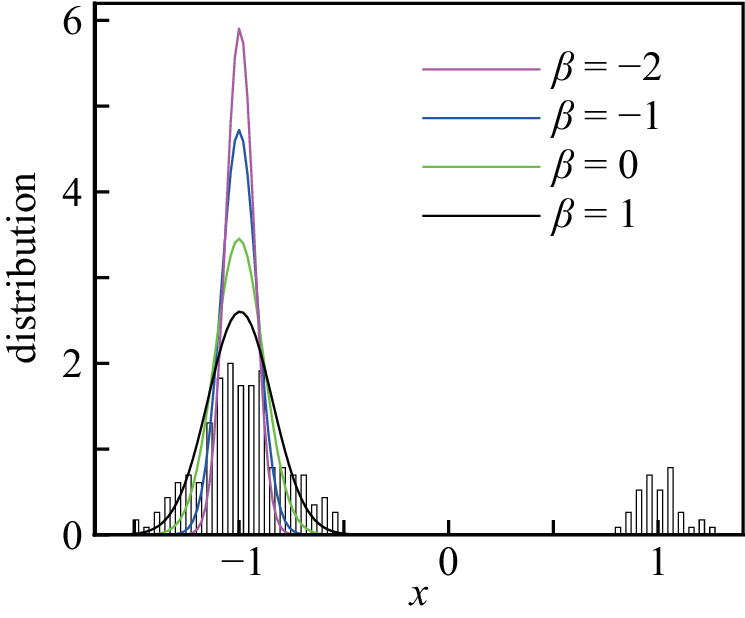}
\caption{EB-RANSAC estimators for $\beta = -2,-1,0,1$. 
The bars represent the data histogram that is the same as that shown in figure \ref{fig:MLE_Gauss}(a).}
\label{fig:MLE_Gauss_beta}
\end{figure}

Figure \ref{fig:MLE_Gauss_beta} shows the EB-RANSAC estimators for different values of $\beta$. 
As $\beta$ decreases, it can be observed that the locality of the EB-RANSAC estimator increases. 
This behavior aligns with the theoretical considerations provided in Section \ref{sec:theory}.

\subsubsection{Exponential distribution}
\label{sec:MLE_Exponential}

\begin{figure}[htb]
\centering
\includegraphics[height=4cm]{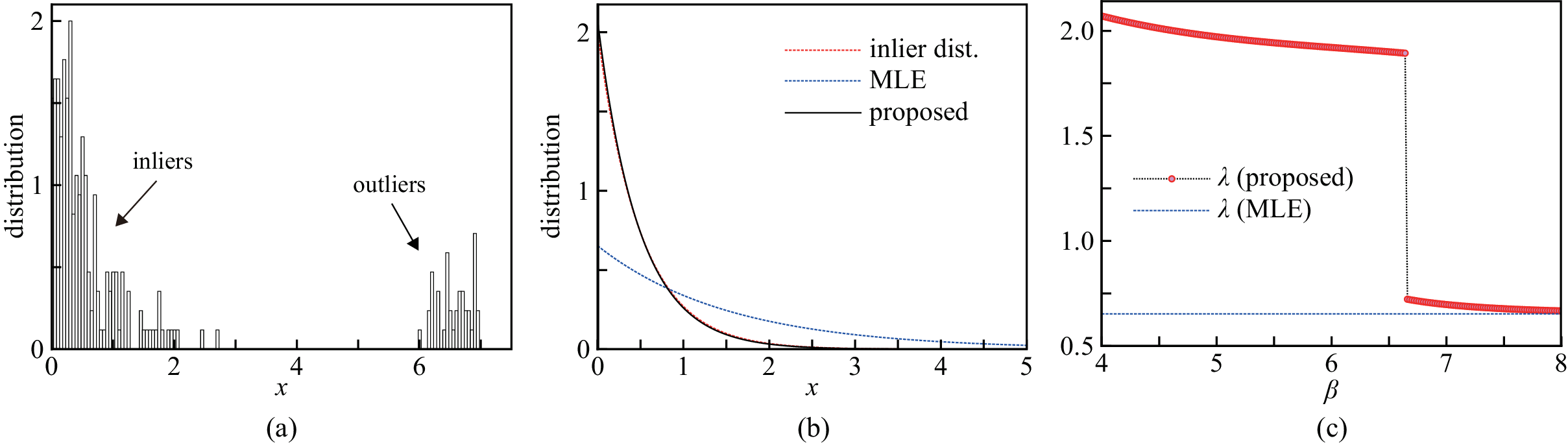}
\caption{(a) Normalized histogram of the dataset. (b) Estimators obtained from the EB-RANSAC (with $\beta = 4$) and the standard MLE methods. 
The red broken line represents the inlier distribution $\mcal{E}(x \mid 2)$. 
(c) Values of $\lambda$ obtained from the EB-RANSAC method for $\beta \in [4,8]$. 
For comparison, the value of $\lambda$ obtained from the standard MLE method ($\lambda \approx 0.653$) also plotted (the blue broken line). 
A jump of $\lambda$ is observed at $\beta_c \approx 6.65$.}
\label{fig:MLE_Exp}
\end{figure}

Consider the MLE of an exponential distribution, $f_{\theta}(x) = \mcal{E}(x \mid \lambda):=\lambda \exp(-\lambda x)$, where $\theta = \{\lambda\}$.
The loss is expressed based on the negative log likelihood, that is the same as that in Section \ref{sec:MLE_Gaussian}.  
Figure \ref{fig:MLE_Exp}(a) depicts the histogram of the dataset 
that comprises the inliers (200 points) drawn from $\mcal{E}(x \mid 2)$ and the outliers (40 points) drawn from $U_{[6,7]}(x)$,  
where $U_{[s,t]}(x)$ denotes a uniform distribution on interval $[s,t]$.
Figure \ref{fig:MLE_Exp}(b) depicts the estimators obtained from the EB-RANSAC and standard MLE methods. 
The EB-RANSAC estimator fits well into the inlier distribution. 

Figure \ref{fig:MLE_Exp}(c) depicts the plot of the value of $\lambda$ obtained using the EB-RANSAC method. 
The discontinuous jump transition of the value of $\lambda$ is observed. 
Similar jump behaviors can be observed in figures  \ref{fig:LMS}(b) and \ref{fig:MLE_Gauss}(b). 
We investigated the landscape of the EB-RANSAC loss to understand the reason for occurrence of such a singular behavior. 
Suppose that the data points are generated from a mixture distribution expressed as follows:
\begin{align*}
q_{\mrm{gen}}(x) := r \mcal{E}(x \mid 2) + (1 - r) U_{[6,7]}(x),
\end{align*}
where $r \in [0,1]$ is the mixture ratio. $q_{\mrm{gen}}(x)$ is the mixture of the inlier and outlier distributions. 
$r = 200/240$ is fixed to adjust to the experiment in this section. 
When the size of the dataset that is generated from $q_{\mrm{gen}}(x)$ is sufficiently large, 
the EB-RANSAC loss in equation (\ref{eqn:EBM-RANSAC-Loss}) will converge to 
\begin{align*}
\lim_{N \to \infty}L_{\mrm{ER}}(\theta; D, \beta) = L_{\mrm{ER}}(\theta; \beta) := -\int_0^{\infty} q_{\mrm{gen}}(x)\sfp \big(\beta - \ell(\theta; x) \big)dx 
\end{align*}
based on the law of large numbers, where $\ell(\theta; x) = - \ln \mcal{E}(x \mid \lambda)$.
Consider a modified EB-RANSAC loss:
\begin{align}
\hat{L}_{\mrm{ER}}(\theta; \beta):=  L_{\mrm{ER}}(\theta; \beta) + \sfp(\beta).
\label{eqn:modify_EBM-RANSAC-Loss}
\end{align}
In the modified EB-RANSAC loss, the second term, $\sfp(\beta)$, modifies the base lines of $L_{\mrm{ER}}(\theta; \beta)$ for different values of $\beta$. 
The landscape of the modified EB-RANSAC loss is shown in figure \ref{fig:landscape}(a). 
Figure \ref{fig:landscape}(b) illustrates the scheme of the jump transition. 
The jump transition occurs when the minimum point is switched. 
Such a jump transition is also observed, for instance, in the Potts model in statistical mechanics~\cite{Mittag1974}. 
Because $\ell(\theta; x) = -\ln \lambda + \lambda x$ is convex with respect to $\lambda$, the minimization of equation (\ref{eqn:loss_function}) 
is a convex optimization problem. However, from figure \ref{fig:landscape}, the EB-RANSAC loss is not generally convex. 
The EB-RANSAC method breaks the convex structure of the original minimization problem.

\begin{figure}[htb]
\centering
\includegraphics[height=4cm]{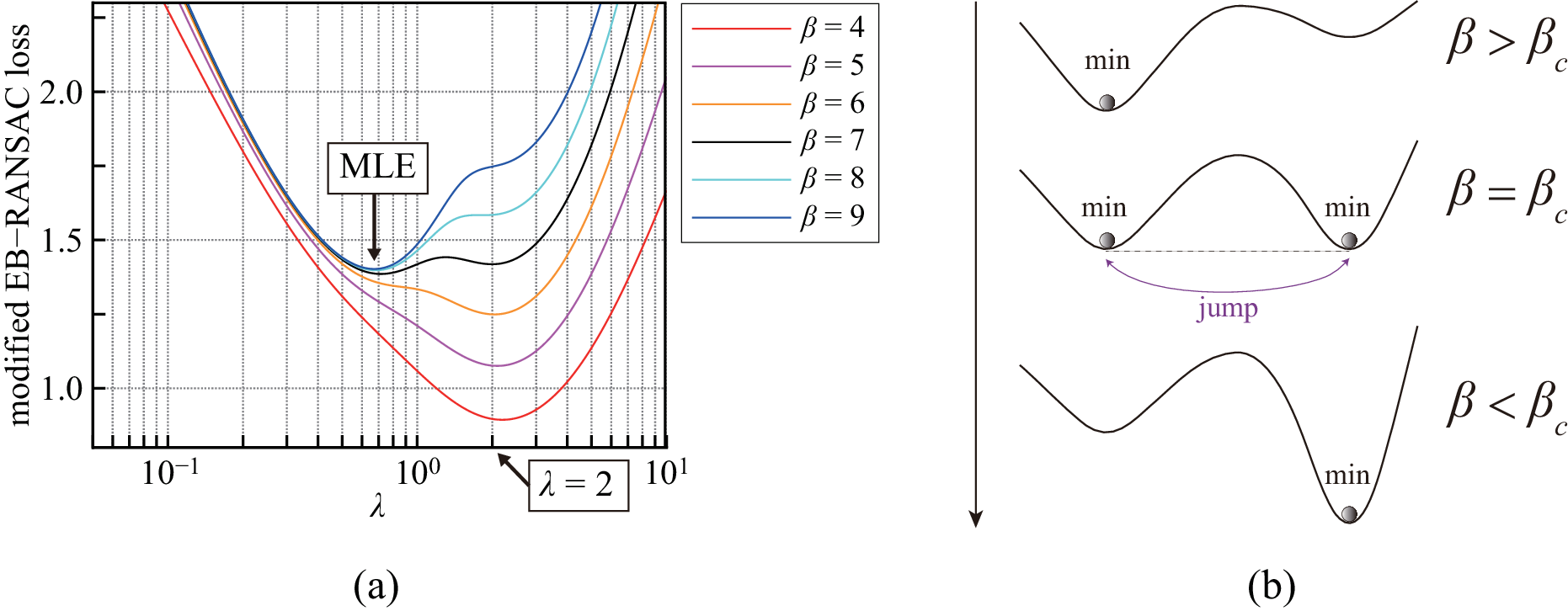}
\caption{(a) Landscapes of the modified EB-RANSAC loss in equation (\ref{eqn:modify_EBM-RANSAC-Loss}) for various values of $\beta$.  
(b) Schematic of the jump behavior.  At $\beta = \beta_c$, the minimum point is switched, which causes the jump transition.}
\label{fig:landscape}
\end{figure}

\section{Conclusion and Future Study}
\label{sec:conclusion}

In this study, based on the idea of RANSAC, we proposed a probabilistic model based on an EBM: EB-RANSAC. 
As elucidated in Section \ref{sec:conditional-distribution_RBM}, the proposed joint distribution in equation (\ref{eqn:proposed_model-joiint}) has a similar scheme to that of RANSAC. 
Through the marginal operation for the joint distribution, the proposed EB-RANSAC method in equation (\ref{eqn:EBM-RANSAC}) is obtained 
as a form of the deterministic minimization problem of the EB-RANSAC loss. 
Owing to the marginalization, EB-RANSAC does not explicitly require a sampling procedure unlike RANSAC, and therefore, it has repeatability. 
Furthermore, the number of the hyperparameter in EB-RANSAC is only one (by contrast, RANSAC has several hyperparameters).
Equation (\ref{eqn:EBM-RANSAC}) is simple 
and it can be applied to any estimation problems with a loss function in the form of equation (\ref{eqn:loss_function}), for instance, the learning problem of neural networks. 

Two important issues remain in the EB-RANSAC framework. 
The first issue is the tuning of the hyperparameter $\beta$. 
$\beta$ corresponds to the threshold  $T_{\mrm{cons}}$ in RANSAC as elucidated in Section \ref{sec:conditional-distribution_RBM}, 
and it determines the cut-off threshold $T_{\mrm{cut}}(\beta)$ as discussed in Section \ref{sec:theory}. 
As the numerical experiments in Section \ref{sec:experiments} demonstrated, a suitable value (or range) of $\beta$ exists. 
An appropriate value of $\beta$ is required to achieve a desired robust estimation. 
However, an optimization of $\beta$ would not be trivial. 
As discussed in Section \ref{sec:MLE_Exponential}, the landscape of the EB-RANSAC loss has a relation to the picture of the phase transition in statistical mechanics. 
We expect that analyzing this relationship using a Bayesian approach, such as an empirical Bayes~\cite{MacKay1992} or hierarchical Bayes approaches~\cite{Bishop2006}, 
will provide insights into the $\beta$-optimization. 

The other issue is the minimization of the EB-RANSAC loss. 
As discussed in Section \ref{sec:MLE_Exponential}, the EB-RANSAC loss may not be a convex function with respect to $\theta$ 
when the original loss function in equation (\ref{eqn:loss_function}) is convex. 
Therefore, a global minimization algorithm for the EB-RANSAC loss is important. 
Recently, a method that can achieve a global optimization on restricted Boltzmann machines (RBMs) has been proposed~\cite{Sekimoto2025}. 
An RBM is a probabilistic model with a two-layer structure: the visible and hidden layers. 
In \cite{Sekimoto2025}, the density maximization (i.e., the energy minimization) problem on the marginal distribution over the visible layer is considered. 
Similar to an RBM, the proposed EBM in equation (\ref{eqn:proposed_model-joiint}) has a two-layer structure: $\theta$- and $\bm{w}$-layers. 
We expect the method proposed in \cite{Sekimoto2025} to contribute to the global minimization problem for the EB-RANSAC loss.

\section*{Acknowledgments}
This work was partially supported by Japan Society for the Promotion of Science (JSPS) KAKENHI: Grant Number JP24KJ0452. 
We would like to thank K. Sakai and C. Tsuchiya for invaluable discussions.

\appendix

\section{Proofs of Theorems}

In this Appendix, the proofs of the theorems presented in Section \ref{sec:theory} are provided.

\subsection{Proof of Theorem \ref{thm:minimum_solution}}
\label{app:proof_minimum-solution}

We denote the set of the realizations of $\bm{x}$ as $\mcal{X}:= \{\xi_k \mid k = 1,2,\ldots, K\}$.
Therefore, the EB-RANSAC loss in equation (\ref{eqn:EBM-RANSAC_MLE}) is rewritten as
\begin{align*}
L_{\mrm{ER}}(\theta; D, \beta) = -\sum_{\bm{x} } q(\bm{x})\ln \big( 1 + e^{\beta}  p_{\theta}(\bm{x}) \big)
=-\sum_{k=1}^K q(\xi_k) \ln \big( 1 + e^{\beta}  \rho_k \big),
\end{align*}
where $\rho_k := p_{\theta}(\xi_k) \geq 0$. 
The Hessian matrix of $L_{\mrm{ER}}(\theta; D, \beta)$ with respect to $\bm{\rho}:=\{\rho_k \mid  k = 1,2,\ldots, K\}$ is semi-positive definite, 
because
\begin{align*}
\frac{\partial^2 L_{\mrm{ER}}(\theta; D, \beta)}{\partial \rho_k \partial \rho_{\ell}}=
\begin{dcases}
\frac{q(\xi_k) e^{2\beta}}{(1 + e^{\beta}  \rho_k )^2}\geq 0 & (k = \ell)\\
0 & (k \neq \ell)
\end{dcases}
.
\end{align*}
This implies that the EB-RANSAC loss is convex with respect to $\bm{\rho}$. 

The minimization problem is reformulated as
\begin{align*}
\min_{\bm{\rho}}L_{\mrm{ER}}(\theta; D, \beta)\quad \mrm{s.t.} \quad 
\sum_{k=1}^K \rho_k = 1, \>\> \rho_k \geq 0.
\end{align*}
We introduce the Lagrange multipliers to solve this minimization problem,  
in which the Lagrangian is 
\begin{align*}
L(\bm{\rho}, \zeta, \bm{\nu}):=L_{\mrm{ER}}(\theta; D, \beta) + \zeta \Big( \sum_{k=1}^K \rho_k - 1\Big)
-\sum_{k=1}^K \nu_k \rho_k.
\end{align*}
The second and third terms of this Lagrangian correspond to the multipliers for the normalization constraint, $\sum_{k=1}^K \rho_k  = 1$, 
and nonnegative constraint, $\rho_k\geq 0$, respectively. 
The extremal condition $\partial L(\bm{\rho}, \zeta, \bm{\nu}) / \partial \rho_k = 0$ leads to
\begin{align}
\rho_k= \frac{q(\xi_k)}{\zeta- \nu_k} - e^{-\beta}.
\label{eqn:extremal-L_wrt_P}
\end{align}
Here, we consider the Karush--Kuhn--Tucker (KKT) condition: (i) $\rho_k \geq 0$ when $\nu_k = 0$
and (ii) $\nu_k \geq  0$ when $\rho_k = 0$. 
From condition (i), 
\begin{align}
\rho_k= \relu\Big(\frac{q(\xi_k)}{\zeta} - e^{-\beta}\Big)
\label{eqn:extremal-L_wrt_P_relu}
\end{align}
is obtained, where $\zeta$ is determined to satisfy the normalization constraint as follows: 
\begin{align}
\sum_{k=1}^K \relu\Big(\frac{q(\xi_k)}{\zeta} - e^{-\beta}\Big) = 1.
\label{eqn:equation_zeta}
\end{align}
As $q(\xi_k) \geq 0$, $\zeta > 0$ is ensured. 
Alternatively, from equation (\ref{eqn:extremal-L_wrt_P}), we have
\begin{align}
\nu_k=\zeta - \frac{q(\xi_k)}{\rho_k + e^{-\beta}}.
\label{eqn:nu_k}
\end{align}
From equation (\ref{eqn:extremal-L_wrt_P_relu}), $q(\xi_k) \leq e^{-\beta}\zeta$ when $\rho_k = 0$. 
This implies that $\nu_k \geq 0$ when $\rho_k = 0$, and therefore, condition (ii) is confirmed. 
Thus, equations (\ref{eqn:extremal-L_wrt_P_relu}) and (\ref{eqn:nu_k}) satisfy the KKT condition. 
Furthermore, the fact that the EB-RANSAC loss is convex with respect to $\bm{\rho}$ 
and all constraints are linear with respect to $\bm{\rho}$ ensures that the obtained solution is the global minimum (i.e., Slater's condition~\cite{Slater2014}).

For a nonnegative constant $c$, $\relu(cx) = c\relu(x)$,   
and therefore, equations (\ref{eqn:extremal-L_wrt_P_relu}) and (\ref{eqn:equation_zeta}) can be rewritten as
\begin{align*}
\rho_k= \frac{1}{\zeta}\relu (q(\xi_k) - e^{-\beta}\zeta)\quad \mrm{and}
\quad \zeta = \sum_{k=1}^K \relu(q(\xi_k) - e^{-\beta}\zeta),
\end{align*}
respectively. By setting $T_{\mrm{cut}}(\beta) = e^{-\beta}\zeta$, we arrive at Theorem \ref{thm:minimum_solution}.

\subsection{Proof of Theorem \ref{thm:T(beta)}}
\label{app:proof_T(beta)}

Before presenting the proof of Theorem \ref{thm:T(beta)}, we first show the following Lemma:
\begin{lemma}
$b(T)$ is a continuous and monotonic decreasing function for $T \in \mbb{R}$.
\label{lemma:b(T)}
\end{lemma}
\begin{proof}
Suppose that the realizations of $\bm{x}$, $\mcal{X}= \{\xi_k \mid k = 1,2,\ldots, K\}$, are sorted 
to satisfy $q_1 \leq q_2 \leq \cdots \leq q_K = T^*$, where $q_k := q(\xi_k)$. 
For $T \in [q_k, q_{k+1}]$ for $1 \leq k \leq K-1$, 
\begin{align*}
b(T) = \sum_{s=1}^K \relu (q_{s} - T) =  \sum_{s=k+1}^K (q_{s} - T)=(1 - Q_k) - (K - k)T,
\end{align*}
where $Q_k := \sum_{s=1}^k q_s$. 
Hence, $b(T)$ is a continuous and monotonic decreasing function for $T \in [q_k, q_{k+1}]$ because it is a linear function for $T$.
For $T < q_1$, $b(T)$ is also a continuous and monotonic decreasing function because $b(T) = 1 - K T$.
For $T > q_K$, $b(T)$ is a constant: $b(T) = 0$. 
In addition, $b(T)$ is continuous at $T = q_k$ for $k = 1,2,\ldots, K$,  
because $\lim_{T \to q_k +0}b(T)  = (1 - Q_k) - (K - k)q_k$
and $\lim_{T \to q_k -0}b(T)  = (1 - Q_{k-1}) - (K - {k-1})q_k = (1 - Q_k) - (K - k)q_k$. Here, $Q_0 = 0$ is assumed.
\end{proof}

From Lemma \ref{lemma:b(T)}, $h_{\beta}(T)$ is a continuous and monotonic increasing function. 
Because $b(0) =  1$ and $b(T) = 0$ for $T \geq T^*$, 
$h_{\beta}(0) = -e^{-\beta}b(0) = -e^{-\beta}< 0$ and $h_{\beta}(T^*) = T^* > 0$. 
Thus, $h_{\beta}(T)$ must have a unique root in the range $(0, T^*)$ from the intermediate value theorem. 
Therefore, the proof of Theorem \ref{thm:T(beta)}(i) is obtained.

For $\beta_0 < \beta_1$, we have 
\begin{align*}
h_{\beta_1}(T_{\mrm{cut}}(\beta_0)) &=  T_{\mrm{cut}}(\beta_0) - e^{-\beta_1}\sum_{\bm{x}}\relu \big( q(\bm{x}) - T_{\mrm{cut}}(\beta_0) \big)\nn
&> T_{\mrm{cut}}(\beta_0) - e^{-\beta_0}\sum_{\bm{x}}\relu \big( q(\bm{x}) - T_{\mrm{cut}}(\beta_0) \big) = h_{\beta_0}(T_{\mrm{cut}}(\beta_0)) = 0.
\end{align*}
Therefore, from the intermediate value theorem and Theorem \ref{thm:T(beta)}(i), 
$ h_{\beta_1}(T) $ has a unique root in the range $(0, T_{\mrm{cut}}(\beta_0))$, implying that $T_{\mrm{cut}}(\beta_0) > T_{\mrm{cut}}(\beta_1)$.  
The proof of Theorem \ref{thm:T(beta)}(ii) is obtained.

From Theorem \ref{thm:T(beta)}(ii), $T_{\mrm{cut}}(\beta)$ is a monotonic function for $\beta$. 
Hence, its the inverse function can be defined as
\begin{align*}
\beta(T) = \ln b(T) - \ln T, 
\end{align*}
which is obtained by solving $h_{\beta}(T) = 0$ with respect to $\beta$. 
$\lim_{T \to 0}\beta(T) = \infty$ and $\lim_{T \to T^*}\beta(T) = -\infty$ 
leads to Theorem \ref{thm:T(beta)}(iii).

\bibliographystyle{unsrt}
\bibliography{citation}

\begin{thebibliography}{10}

\bibitem{Menezes2021}
D.~Q.~F. de~Menezes, D.~M. Prata, A.~R. Secchi, and J.~C. Pinto.
\newblock A review on robust {M}-estimators for regression analysis.
\newblock {\em Computers \& Chemical Engineering}, 147(8):107254, 2021.

\bibitem{Fischler1981}
M.~A. Fischler and R.~C. Bolles.
\newblock Random sample consensus: a paradigm for model fitting with
  applications to image analysis and automated cartography.
\newblock {\em Communications of the ACM}, 24(6):381--395, 1981.

\bibitem{Yasuda2025}
M.~Yasuda, N.~Watanabe, and K.~Sekimoto.
\newblock {RANSAC}-based {P}robabilistic {M}odel for {R}obust {E}stimation.
\newblock {\em In Proc. of the 2025 International Symposium on Nonlinear Theory
  and Its Applications}, pages 969--999, 2025.

\bibitem{Raguram2013}
R.~Raguram, O.~Chum, M.~Pollefeys, J.~Matas, and J.~Frahm.
\newblock {USAC}: A universal framework for random sample consensus.
\newblock {\em IEEE Transactions on Pattern Analysis and Machine Intelligence},
  35(8):2022--2038, 2013.

\bibitem{Chum2003}
O.~Chum, J.~Matas, and J.~Kittler.
\newblock Locally optimized {RANSAC}.
\newblock {\em In DAGM-Symposium 2003}, pages 236--243, 2003.

\bibitem{Wang2017}
Y.~Wang, A.~Kucukelbir, and D.~M. Blei.
\newblock Robust {P}robabilistic {M}odeling with {B}ayesian {D}ata
  {R}eweighting.
\newblock {\em In Proc. of the 34th International Conference on Machine
  Learning}, 70:3646--3655, 2017.

\bibitem{Kirkpatrick1983}
S.~Kirkpatrick, C.~D.~Jr. Gelatt, and M.~P. Vecchi.
\newblock {O}ptimization by {S}imulated {A}nnealing.
\newblock {\em Science}, 220(4598):671--680, 1983.

\bibitem{Chum2005}
O.~Chum and J.~Matas.
\newblock Matching with {PROSAC} - progressive sample consensus.
\newblock {\em In Proc. of the 2005 IEEE Computer Society Conference on
  Computer Vision and Pattern Recognition}, 1:220--226, 2005.

\bibitem{Mittag1974}
L.~Mittag and M.~J. Stephen.
\newblock Mean-field theory of the many component {P}otts model.
\newblock {\em Journal of Physics A: Mathematical, Nuclear and General},
  7(9):L109--L112, 1974.

\bibitem{MacKay1992}
D.~J.~C. MacKay.
\newblock Bayesian {I}nterpolation.
\newblock {\em Neural Computation}, 4(3):415--447, 1992.

\bibitem{Bishop2006}
C.~M. Bishop.
\newblock {\em Pattern {R}ecognition and {M}achine {L}earning}.
\newblock Springer-Verlag New York, 2006.

\bibitem{Sekimoto2025}
K.~Sekimoto and M.~Yasuda.
\newblock {M}arginal {P}robability {M}aximization on {R}estricted {B}oltzmann
  {M}achines based on {B}locked {S}ampling.
\newblock {\em In Proc. of the 32nd International Conference on Neural
  Information Processing}, pages 405--418, 2025.

\bibitem{Slater2014}
M.~Slater.
\newblock {\em {L}agrange {M}ultipliers {R}evisited}, pages 293--306.
\newblock Springer Basel, 2014.

\end{thebibliography}

\end{document}